\documentclass[11pt, a4paper, copyright]{simular}

\usepackage{amsmath,amsfonts,bm}

\def\eqref#1{equation~\ref{#1}}

\def\1{\bm{1}}

\DeclareMathAlphabet{\mathsfit}{\encodingdefault}{\sfdefault}{m}{sl}
\SetMathAlphabet{\mathsfit}{bold}{\encodingdefault}{\sfdefault}{bx}{n}

\usepackage{tcolorbox}
\usepackage{fvextra}
\usepackage{wrapfig}
\usepackage{multirow}
\usepackage{subcaption}
\usepackage{xspace}

\theoremstyle{plain}
\newtheorem{proposition}{Proposition}
\newtheorem{lemma}{Lemma}

\newif\ifcomments
\commentstrue   

\newcommand{\method}{\emph{LEASH}}
\newcommand{\pass}{\textup{\textsc{Pass}}}
\newcommand{\fail}{\textup{\textsc{Fail}}}

\definecolor{ourscolor}{RGB}{245,239,252}
\definecolor{opuscolor}{RGB}{235,244,252}

\title{LEASH: Uncertainty-Aware Weak-to-Strong Escalation and Oversight for Computer-Use Agents} 

\author[1]{Qingni Wang\textsuperscript{\dag}}
\author[2]{Changdae Oh}
\author[1]{Yuzhe Yang\textsuperscript{\dag}}
\author[2]{Sharon Li}
\author[1]{Jiachen Yang}
\author[1]{Ang Li}
\author[1]{Xin Eric Wang}
\affil[1]{Simular}
\affil[2]{University of Wisconsin--Madison}
\makeatletter
\g@addto@macro\AB@affillist{\\[3pt]\Affilfont\textsuperscript{$\dagger$}Work done while interning at Simular.}
\makeatother

\begin{abstract}
Computer-use agents increasingly tackle long-horizon tasks, where running a frontier model at every step is expensive and a small open-weight model alone often fails. \emph{Weak-to-strong escalation} offers a middle ground where a weak agent acts by default and hands control to a strong one only when needed.
The core challenge is deciding when to escalate, which requires an uncertainty signal cheap enough to check at every step. However, even after handoff, the strong agent can still get lost, but a black-box API that the strong agent is usually served exposes no internals to monitor.
To this end, we propose \method{}, a training-free framework that unleashes the strong agent only when the weak agent's latent uncertainty signals it is needed, and keeps it on a leash afterward.
\method{} has three layers: 
(1) initial task uncertainty for early routing, (2) step-level uncertainty for runtime escalation, and, after handoff, (3) a shadow observer for continued oversight of the strong black-box agent.
Here, we use \emph{latent uncertainty} as a fuel for all three layers. Specifically, we compare the weak model's hidden states with those from its past successful and failed trajectories, and a state that looks more like the failures signals higher uncertainty.
During weak-agent execution, our uncertainty estimator reuses activations already produced during action generation, so it needs no extra forward pass or dedicated monitor training.
Experiments on OSWorld across four contemporary model backbones show that \method{} consistently improves over standalone weak agents and outperforms random escalation by 6.55--11.94 points at matched escalation rates. With Qwen3.5-35B-A3B, \method{} achieves 73.66\% task success at \$0.4117 per task, matching always-on Claude Opus 5 (72.90\%) while reducing inference cost by 19.2\%.
\end{abstract}

\usepackage[authoryear, sort&compress, round]{natbib}
\uselogo{}
\renewcommand{\today}{}

\begin{document}
\maketitle

\begin{figure}[h]
    \centering
    \includegraphics[width=0.85\linewidth]
    {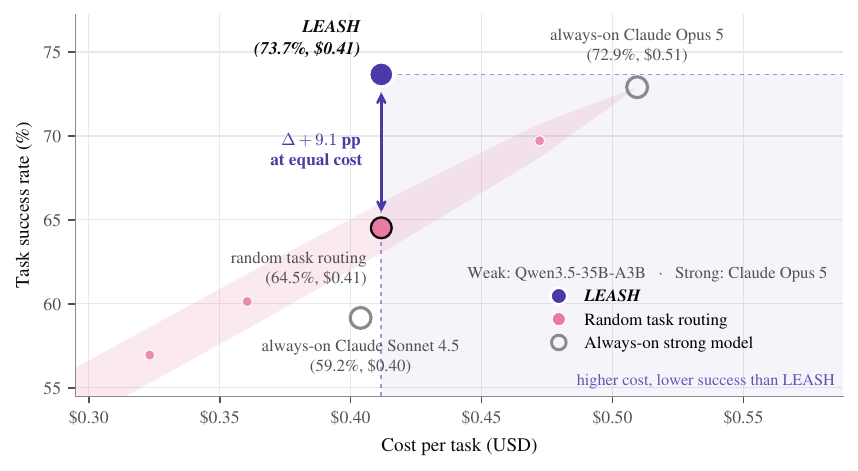}
    \caption{Accuracy--cost comparison between \method{} and random task-level routing on Qwen3.5-35B-A3B. The area spanned by random routing is obtained from repeated random task allocations across different escalation budgets; the band shows $\pm1$ standard deviation over 10,000 allocations. At matched inference cost, \method{} achieves 9.1 percentage points higher task success.}
    \label{fig:random_frontier}
\end{figure}

\section{Introduction}
Computer-use agents (CUAs) automate tasks of graphical interfaces by translating natural-language instructions and visual observations into executable actions~\citep{xie2024osworld,cheng2024seeclick}, and are increasingly deployed for complex, long-horizon automation tasks that can take more than an hour on average~\citep{abhyankar2026osworld}. While frontier proprietary models have been shown to achieve impressive performance leaps~\citep{openai2026gpt6astra,anthropic2026opus55}, running every single task spanning hundreds of steps with these models raises cost concerns, whereas a light open-weight model alone performs poorly on even much simpler tasks~\citep{xie2024osworld}. 
This creates a fundamental trade-off between accuracy and cost efficiency, motivating \textit{weak-to-strong escalation} (W2SE), in which a weak agent serves as the primary executor and hands control to a strong agent only when needed~\citep{wei2026step,lei2026tacit}.
W2SE raises two important questions: (1) When should the weak agent escalate, and (2) who oversees the strong agent after handing over?

Existing work on W2SE primarily focuses on the first question, estimating the weak agent's failure risk to decide when escalation is needed. Some of them adopt another monitoring agent to verify each step~\citep{wang2024mobile,yi2026ecoagent}, others train an extra monitoring module on the collection of annotated trajectories~\citep{wei2026step,lei2026tacit}. However, the former inevitably adds nontrivial cost to run another heavy agent step-by-step, whereas the latter creates another phase of data annotation and training, hindering direct adoption in real-world applications under varying deployment constraints. While training-free uncertainty estimation-based approaches might avoid the aforementioned bottlenecks~\citep{oh2026uncertainty}, prior works still rely on external knowledge base~\citep{liu2026adaptive}, another reference model~\citep{oh2026neglected}, or multiple sampling for each step~\citep{wang2026safeground} quickly becomes infeasible or impractical in long-horizon tasks in the wild.

Besides, escalation alone does not eliminate execution uncertainty, resulting in the second important question in W2SE which is commonly overlooked. That is, even after a stronger agent takes control, subsequent actions may deviate from the intended task or fail to achieve the desired outcome. However, monitoring a black-box strong agent is challenging because its
internal states are inaccessible and verbalized confidence can be unreliable on evolving context of long-horizon task~\citep{han2024towards,kirchhof2026selfreflect}.

To answer both questions, we formulate W2SE as a budget-constrained sequential stopping problem, and use the weak agent's failure risk as a proxy for the expected benefit of escalation. We then propose \textbf{L}atent uncertainty for \textbf{E}scalation \textbf{A}nd \textbf{SH}adow oversight (\method{}), a training-free framework driven by the uncertainty signal in the weak agent's hidden states. First, we record its hidden states of prompt and action tokens on a small set of held-out tasks and label them with task-level and subgoal-level outcomes to construct a representation bank. This representation bank then serves as a semantic anchor to score the newly faced prompts and actions during test-time by measuring, under a Mahalanobis metric in latent space, how much closer each one lies to past failures than to past successes. 
\method{} puts this signal on two jobs (Figure~\ref{fig:overview}): For \emph{escalation}, an initial gate scores uncertainty from the task prompt query, while a runtime gate monitors subsequent execution trace. For \emph{oversight}, after the weak model hands the control to the strong agent, it switches its role from actor to observer that monitors the strong agent's trajectory every three steps; when its uncertainty is high, the pending action is rolled back and re-decided once with higher reasoning effort.  
Importantly, our uncertainty-driven W2SE formulation draws on an easily accessible resource, the weak agent's own past experience~\citep{silver2025welcome}, removing reliance on an external dataset or additional training without hosting another expensive agent during runtime, while offering extensibility for recursive self-improvement~\citep{schmidhuber2007godel,zelikman2022star,madaan2023self,NEURIPS2023_1b44b878,wang2023voyager}.

We evaluate \method{} on OSWorld~\citep{xie2024osworld} with four open-weight weak agents and Claude Opus~5~\citep{anthropic2026opus5} as the strong agent. At matched escalation rates, \method{} outperforms random escalation by 6.55--11.94 points. We further measure how much of the performance gap between the weak agent alone and always-on Opus is recovered by escalation, following the performance-gap metric of~\citep{burns2023weak}.  \method{} recovers 75--102\% of this gap, against 37--69\% for random escalation, while costing 69--81\% as much as always-on Opus. With Qwen3.5-35B-A3B~\citep{qwen35blog}, \method{} matches always-on Opus (73.7\% vs.\ 72.9\% success) at 19\% lower cost, whereas random escalation at equal cost scores about 9 points lower (Figure~\ref{fig:random_frontier}). 
\begin{figure}[t]
    \centering
    \includegraphics[width=\linewidth]{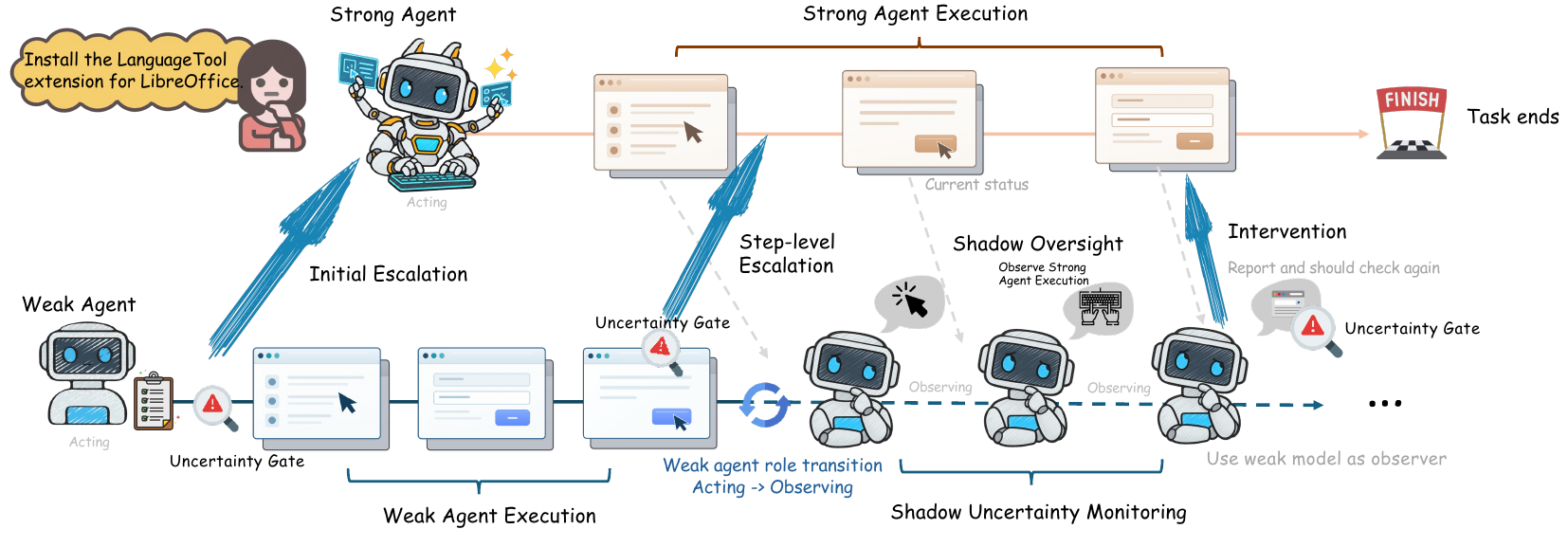}
    \caption{Overview of uncertainty-aware weak-to-strong escalation and post-handoff oversight. A weak agent goes through the task until its uncertainty scores trigger initial or step-level escalation; otherwise, the weak agent continues. After escalation is triggered, the strong agent retains control while the weak agent changes its role from an actor to a shadow observer for continued monitoring. The feedback path from weak observer to strong actor enables uncertainty-triggered reflection and intervention. In short, \method{} unleashes the strong agent only when the weak agent's latent uncertainty signals it is needed, and keeps it on a leash afterward.}
    \label{fig:overview}
\end{figure}

Our contributions are summarized as follows:
\begin{itemize}
    \item We formalize weak-to-strong escalation as a budget-constrained sequential stopping problem and show that one latent uncertainty signal can drive both escalation and oversight.
    \item We propose a training-free, scalable uncertainty scoring method that measures how much closer a hidden state lies to the weak agent's own past failures than to its past successes, and provide theoretical justification for using this score for escalation.
    \item We extend this uncertainty signal to weak-to-strong oversight, in which the white-box weak agent monitors the black-box strong agent through its own hidden states.
    \item On OSWorld with four open-weight backbones, we show that \method{} outperforms random escalation at matched escalation rates and recovers 75--102\% of the escalation gap at 69--81\% of the cost of using always strong agent.
\end{itemize}

\section{Related Work}
\paragraph{Efficient Computer-Use Agents.}
Model routing and cascading offer a promising way to balance inference cost and performance, as demonstrated by FrugalGPT and RouteLLM~\citep{chen2023frugalgpt,ong2025routellm,pmlr-v267-dekoninck25a}.
Extending this idea to GUI automation, Adaptive VLM Routing selects models based on action difficulty and confidence~\citep{liu2026adaptive}, while StepWise trains lightweight monitors to detect execution stalls and identify milestones for escalation~\citep{wei2026step}.
Our work instead leverages the backbone agent's internal representations to guide both task-level routing and step-level escalation, without training a dedicated monitoring network.

\paragraph{Weak-to-Strong Paradigms.}
Weak-to-strong generalization studies how stronger models learn from weaker supervision~\citep{burns2023weak,somerstep2025transfer,medvedev2025weak,xue2025representations,jung2025trust}. These works primarily concern training-time transfer of supervision or capabilities from a weaker model to a stronger one.
Here we study the complementary inference-time setting of \emph{weak-to-strong escalation} (W2SE), where a weak model executes by default and selectively invokes a stronger model when additional capability is needed.
This shifts the weak--strong relationship from supervision transfer to sequential resource allocation. We derive escalation signals from the agent's own past experience without training an additional monitor.

\paragraph{Uncertainty Estimation and Agent Monitoring.}
Uncertainty estimation of agentic systems has emerged as one of the most important problems to monitor and guide agent decisions~\citep{oh2026uncertainty,chun2026comet,han2024towards}, while SafeGround and Argus investigate uncertainty estimation for GUI grounding, including calibrated predictions and representation-based methods~\citep{wang2026safeground,kumar2026uncertainty}.
Sampling-based uncertainty estimation and risk control have also been studied for question answering~\citep{duan2025uprop,wang2025sample}, but repeated generation can be costly when applied at every step of a long-horizon agent trajectory.
Meanwhile, another line of work explores weak-to-strong oversight of agent misbehavior and failure monitoring through observable trajectories~\citep{pmlr-v235-greenblatt24a,kale2026reliable,pan2026monitoring}.
Rather than focusing on either escalation or oversight alone, we handle both problem in a single unified fueled by a simple uncertainty estimate. That is, we integrate representation-based uncertainty estimation into the execution workflow, reusing the backbone model's hidden states for escalation and extending the same signal to post-handoff oversight of a closed-source strong agent.

\section{Method}
\label{sec:method}

We introduce \method{}, a training-free uncertainty-driven agentic inference framework for W2SE and post-handoff oversight (Figure~\ref{fig:overview}). By comparing the weak model's runtime hidden states with successful and failed reference executions, we estimate uncertainty at task initialization and during execution, which guides both escalation to the strong agent and continued oversight after handoff. We first formulate the problem in Section~\ref{sec:problem} and introduce each component in subsequent sections.

\subsection{Problem Formulation}
\label{sec:problem}

Given a task query, a computer-use agent generates an action conditioned on the query, the screen, and history. We consider a lightweight, weak agent, e.g., a relatively small open-weight model with accessible hidden states, and a large, strong agent, e.g., a proprietary black-box frontier model, sharing a budget of $T$ actions. In addition to these two actor agents, we also host a Manager agent (with a strong model) who fires at the beginning to decompose the task and subsequently at subgoal boundaries to verify progress, constructing a minimal version of multi-agent workflow~\citep{agashe2025agents,agashe2025agents2,koh2026multi}. Throughout execution, the Manager decomposes the task query into subgoals with screenshot-verifiable completion criteria and decides whether to advance, keep, or revise each subgoal (Appendix~\ref{app:prompts}). This decomposition naturally provides finer-grained supervision~\citep{pmlr-v306-chun26b}, allowing execution steps to inherit subgoal-level outcomes rather than relying solely on the final task-level label.

We adopt the weak model as a starting agent for all tasks until control is transferred to the strong model by hitting some flags. Then, the strong model runs over the remaining action budget until the end of the task. Let $x\sim\mathcal{D}$ denote a task drawn from the task distribution, and let $\pi$ denote an escalation policy. For each task $x$, policy $\pi$ induces a stopping time $t_{\mathrm{esc}}^\pi(x)\in\{0,1,\ldots,T\}$, where $t_{\mathrm{esc}}^\pi(x)=0$ denotes the earliest escalation opportunity and $t_{\mathrm{esc}}^\pi(x)=T$ denotes no escalation. Since the strong agent dominates per-step cost, we use the total escalation rate, $p_{\mathrm{esc}}$, as a budget proxy and want to maximize task success under limited budget:
\begin{equation}
\max_{\pi}\;
\mathbb{E}_{x\sim\mathcal{D}}\!\left[S^\pi(x)\right]
\quad\text{s.t.}\quad
\Pr_{x\sim\mathcal{D}}\!\left(t_{\mathrm{esc}}^\pi(x)<T\right)
\le p_{\mathrm{esc}},
\label{eq:objective}
\end{equation}

\begin{wrapfigure}[15]{r}{0.6\textwidth}
    \centering
    \includegraphics[width=\linewidth]{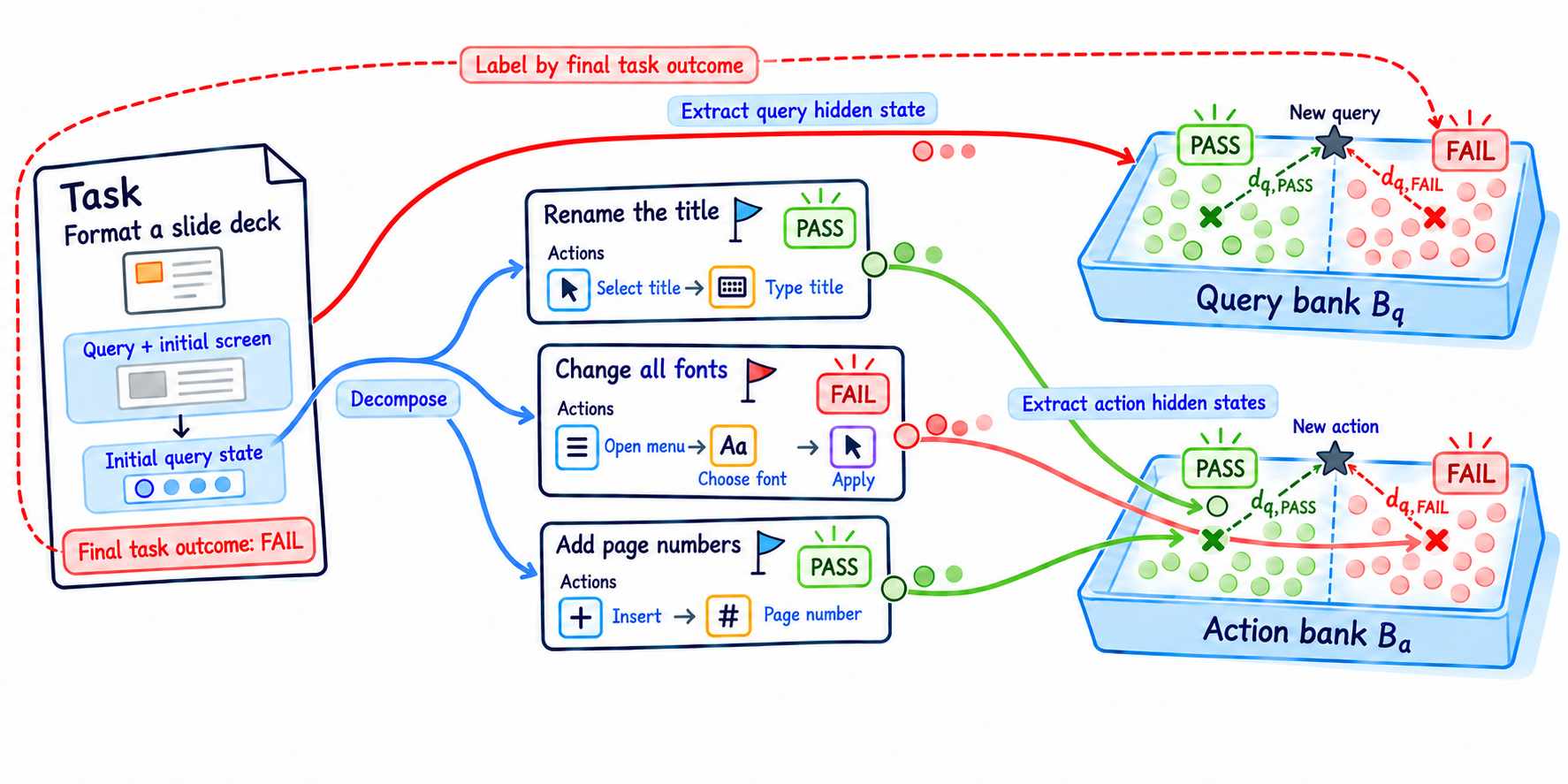}
    \caption{Construction of the query and action reference banks for latent space uncertainty estimation.}
    \label{fig:uncertain}
\end{wrapfigure}

where $S^\pi(x)\in\{0,1\}$ indicates task success under escalation policy $\pi$. That is, we want to run tasks and each step within them through the weak agent as many times as possible unless it fails, to minimize the cost while achieving a high task success rate, becoming a problem \textbf{when-to-escalate}. Unlike existing works that adopt other monitoring agents~\citep{wang2024mobile,gao2026agentcollab} incurring additional inference costs or training-based monitors~\citep{wei2026step} requiring annotation costs and an additional training phase, we propose to use simple uncertainty estimates as a signal to determine when-to-escalate, without inducing any training cost.

\subsection{Reference-Based Latent Uncertainty Estimation}
\label{sec:uncertainty}

Motivated by well-established observations revealing that LLMs internalize the likelihood of task success or failure in their representations~\citep{kadavath2022language,azaria2023internal,kossen2024semantic,orgad2025llms} even before generating any tokens~\citep{ji2024llm,cencerrado2025no,afzal-etal-2025-knowing}, we propose a training-free, representation-distance-based uncertainty estimation technique for agent systems.

\paragraph{Representations and reference banks.}  
For each task from a small held-out set, we run a weak agent to collect some preliminary inference trajectories without escalating anything to a strong agent. During this inference over the held-out set, we extract two final-layer representations, $h_0\in\mathbb{R}^{d}$ at the last task query prompt token, and $h_t\in\mathbb{R}^{d}$ at the last generated action token of each execution step $t$. The former captures information over the instruction and initial screen before decoding, whereas the latter also incorporates the action executed by the agent. 
These hidden states build representation banks pairing each hidden state with progress outcome:
The query bank $\mathcal{B}_q=\{(h_{i,0},y_i)\}$ adopts task outcomes $y_i\in\{\pass,\fail\}$, whereas the action bank $\mathcal{B}_a=\{(h_{i,t},y_{i,t})\}$ assigns each step its corresponding subgoal segment's outcome (assessed by a Manager agent):
\pass{} if the subgoal is successfully completed, \fail{} if it should be revised, and otherwise the final task outcome.

\paragraph{Uncertainty score.} For each channel of representation banks $c\in\{q,a\}$, i.e., query and action, we fit a shared-covariance Gaussian discriminant model~\citep{Mukhoti_2023_CVPR} on $\mathcal{B}_c$. We first standardize each representation with the bank statistics, $z_c(h)=(h-\bar{h}_c)/s_c$, where $\bar{h}_c,s_c\in\mathbb{R}^d$ are the feature-wise mean and standard deviation of all hidden states in $\mathcal{B}_c$ (pooled over both outcomes) and the division is element-wise. Let $\mu_{c,y}$ be the centroid of the standardized representations in $\mathcal{B}_c$ with outcome $y$, and $\Sigma_{c,\lambda}=\widehat{\Sigma}_c+\lambda_c I$ the pooled within-class covariance regularized by $\lambda_c>0$. Using the squared Mahalanobis distance,
\begin{equation}
d_{c,y}^{2}(h) = \bigl(z_c(h)-\mu_{c,y}\bigr)^{\top} \Sigma_{c,\lambda}^{-1} \bigl(z_c(h)-\mu_{c,y}\bigr), \label{eq:reference_distance}
\end{equation}
we define uncertainty score through the relative distance to successful and failed references:
\begin{equation}
u_c(h) = d_{c,\pass}^{2}(h) - d_{c,\fail}^{2}(h). \label{eq:uncertainty_score}
\end{equation}
A larger score indicates greater relative proximity to failed references. We use $u_q$ for initialization and $u_a$ for runtime action monitoring, selecting $\lambda_c$ for each model backbone and channel by task-grouped cross-validated AUROC (detailed in Appendix~\ref{app:method_impl}).

\paragraph{Interpretation.} 
The design of our score was largely motivated by distance-based out-of-distribution detection literature~\citep{lee2018simple,sun2022out,ren2023outofdistribution,muller2025mahalanobis}, but with careful customization for better fit in step-level agent uncertainty estimation. Unlike the relative Mahalanobis distance of \citet{ren2021simple}, which subtracts a class-agnostic background distance, we subtract the distance to the opposite outcome to capture a sharp delta signal of failure. Note that the Gaussian fitting only requires per-outcome means and a singular value decomposition, without gradient-based training, and can be iterated as the bank grows through multiple rounds of experience~\citep{NEURIPS2023_1b44b878,ouyang2026reasoningbank,lee2026recursive}.

\subsection{Uncertainty-Aware Weak-to-Strong Escalation}
\label{sec:escalation}

\paragraph{Escalation gates.}
With the measure of uncertainty, we define escalation gates for each uncertainty channel. Two channels address different sources of uncertainty: ambiguity in task specification and failure likelihood across stacked interactions~\citep{oh2026uncertainty}. The initial query-level routing can identify tasks beyond the weak agent's capability in advance, but cannot anticipate every difficulty encountered during runtime execution. Therefore, we combine an initial escalation gate, $g_0 = \mathbb{I}\bigl\{u_q(h_0)>\tau_{q}\bigr\}$ 
which routes when $g_0=1$, with a runtime escalation gate, $g_t = \mathbb{I}\bigl\{u_a(h_t)>\tau_{a}\bigr\}$ for $t\ge1$ that monitors each action of tasks retained by the initial gate given calibrated thresholds $\tau_{q}$ and $\tau_{a}$.
Runtime escalation occurs immediately after the first action with $g_t=1$; The strong agent then receives the instruction, current screenshot, and recent steps run by the weak agent, and keeps control until the task termination (See Appendix~\ref{app:prompts} for details).

\paragraph{Calibrating gating threshold to an escalation budget.}
We set each threshold to a quantile of uncertainty scores on the reference tasks, so that each gate escalates a target fraction of tasks used to construct the representation bank. Let $p_{q}$ be the target fraction of tasks escalated at initialization and $p_{\mathrm{esc}}$ the target overall escalation rate. Since the runtime gate only monitors the tasks that the initial gate leaves to the weak agent, it should escalate a fraction $p_{a} = (p_{\mathrm{esc}}-p_{q})/(1-p_{q})$ of them, so that $p_{q}+(1-p_{q})p_{a}=p_{\mathrm{esc}}$. The thresholds are then expressed as follows:
\begin{equation}
\tau_{q} = Q_{1-p_{q}}\bigl(\{u_{q}(h_{i,0})\}_{i}\bigr), \qquad \tau_{a} = Q_{1-p_{a}}\bigl(\{U_i : u_{q}(h_{i,0})\le\tau_{q}\}\bigr), \label{eq:thresholds}
\end{equation}
where $Q_{\alpha}$ is the empirical $\alpha$-quantile over reference tasks and $U_i=\max_{1\le t\le T_i}u_{a}(h_{i,t})$ is the maximum runtime score of the $i$-th reference trajectory of length $T_i$. Intuitively, a fraction $p_{q}$ of the reference tasks have initial scores above $\tau_{q}$, and among the tasks kept by the initial gate, a fraction $p_{a}$ have a maximum runtime score above $\tau_{a}$.
Note that we target escalation rates rather than failure probabilities because the gates only use the ranking of the scores. We compute $\tau_{a}$ from each task's maximum score because a single step above $\tau_{a}$ is enough for handoff; a quantile over individual steps would typically escalate more tasks than targeted. 
All scores in Eq.~(\ref{eq:thresholds}) are obtained out-of-fold from the same held-out reference set. After the thresholds are fixed, we refit one final estimator per uncertainty channel on all tasks in this same reference set for evaluation. Because refitting slightly changes the score scale, the target escalation rates are only approximate; we therefore report realized rates and match random baselines to them (Table~\ref{tab:main_results}, Appendix~\ref{app:method_calib}).

\subsection{Post-Handoff Oversight and Intervention}
\label{sec:shadow}

We note that the strong agent can also get stuck in repeated action loops or get lost in multi-turn interactions~\citep{laban2026llms}. However, its hidden states are inaccessible, while verbalized confidence is often highly unreliable as an uncertainty signal~\citep{han2024towards,kirchhof2026selfreflect}.

\paragraph{Shadow oversight.} To handle this challenge, we repurpose the weak agent from an actor as a \emph{Shadow observer} of strong agent for continued post-handoff oversight. We append each executed strong-agent step to the weak model's history in its own action format and query with the current observation every $k$ steps ($k=3$ by default). Given that context, the weak model generates a candidate action, allowing us to extract shadow representation $h^{\rm sh}_t$ at the same action token position as the strong black-box agent. Assuming that a higher uncertainty from the weak model also indicates a greater difficulty for the strong agent on average~\citep{pan2026monitoring,kale2026reliable}, we can use that representation for the runtime action uncertainty estimation again $u_{\rm sh} = u_{a}\bigl(h^{\rm sh}_t\bigr)$, to define the Shadow oversight gate $g_{\rm sh}=\mathbb{I}\bigl\{u_{\rm sh}>\tau_{\mathrm{sh}}\bigr\}$ given a threshold $\tau_{\mathrm{sh}}$.

\paragraph{Uncertainty-triggered intervention.} When $g_{\mathrm{sh}}=1$, we discard the strong agent's pending action and trigger a reflection-based intervention, which prompts the strong model agent to reassess its own progress toward the task goal and choose a new action using higher reasoning effort. We set the threshold $\tau_{\mathrm{sh}}$ to the 90th percentile of per-step reference scores of the action representation bank. 
We allow at most six Shadow interventions per task~\citep{vasudev2026accurate} and terminate further intervention once two have altered the strong agent's proposed action (Appendix~\ref{app:method_shadow}).

\section{Experiments}
\label{sec:experiments}

\subsection{Experimental Setup} \label{sec:setup}

\paragraph{Benchmark and agents.}

We evaluate our framework on OSWorld~\citep{xie2024osworld}. For the weak agent, we use four open-source model backbones: EvoCUA-8B~\citep{xue2026evocua}, Qwen3.5-35B-A3B~\citep{qwen35blog}, Nemotron-3-Nano-Omni (30B-A3B)~\citep{deshmukh2026nemotron}, and Qwen3.8-27B~\citep{qwen38}. For the strong agent that takes over the escalated task and the Manager agent that decomposes tasks into subgoals and verifies their completion, we use Claude Opus~5~\citep{anthropic2026opus5}.
We compare our uncertainty-based W2SE and oversight framework and the baseline against always-on Claude Opus 5 and Claude Sonnet 4.5~\citep{anthropic2025sonnet45}, which use those models for every task and every step. More detailed definitions, execution settings, and model identifiers are provided in Appendix~\ref{app:experimental_details}.

\paragraph{Metrics.}
We report task success rate and mean inference cost per task.
Costs are estimated from recorded token usage at fixed reference prices, including weak agent, Manager, strong-executor, and monitoring calls when enabled. These estimates measure the token-priced inference cost of the overall workflow. 
Cost coverage and accounting details are provided in Appendix~\ref{app:experimental_details}.

\subsection{Task Success and Inference Cost}
\label{sec:main_results}

\begin{table}[t]
\centering
\caption{Task success and estimated inference cost on OSWorld.
\emph{Random} escalation routes a matched fraction of tasks to Opus 5 uniformly at random; success is averaged over 10,000 random allocations. Esc.\% denotes the fraction of tasks handed off through the initial or runtime gate and excludes post-handoff Shadow monitoring. $\Delta$ Success is the absolute improvement over the corresponding standalone agent. Cost is reported in USD per task.
}
\label{tab:main_results}

\small
\renewcommand{\arraystretch}{1.08}
\setlength{\tabcolsep}{3.5pt}

\begin{tabularx}{\linewidth}{@{}lXrrrr@{}}
\toprule
Agent
& Escalation Method
& Esc. (\%)
& Success (\%) $\uparrow$
& $\Delta$ Success
& Cost (\$) $\downarrow$ \\
\midrule

EvoCUA-8B
& -  & 0.0  & 42.37 & --     & -- \\
& Random & 45.4 & 59.98 & +17.61 & 0.3031 \\
\rowcolor{ourscolor}
\cellcolor{white}
& \textbf{\method{}}   & 45.4 & 67.18 & +24.81 & 0.3627 \\
\addlinespace

Qwen3.5-35B-A3B
& -  & 0.0  & 37.40 & --     & -- \\
& Random & 64.4 & 61.72 & +24.32 & 0.3768 \\
\rowcolor{ourscolor}
\cellcolor{white}
& \textbf{\method{}}   & 64.1 & \textbf{73.66} & \textbf{+36.26} & 0.4117 \\
\addlinespace

Nemotron-3
& -  & 0.0  & 45.04 & --     & 0.0132 \\
(30B-A3B)
& Random & 43.3 & 55.41 & +10.37 & 0.2782 \\
\rowcolor{ourscolor}
\cellcolor{white}
& \textbf{\method{}}   & 43.1 & 66.03 & +20.99 & 0.3500 \\
\addlinespace

Qwen3.8-27B
& -  & 0.0  & 51.91 & --     & -- \\
& Random & 51.3 & 63.30 & +11.39 & 0.3249 \\
\rowcolor{ourscolor}
\cellcolor{white}
& \textbf{\method{}}   & 51.1 & 69.85 & +17.94 & 0.3622 \\
\midrule

\rowcolor{opuscolor}
\cellcolor{white}
Claude Opus 5
& Always on & 100.0 & 72.90 & -- & 0.5094 \\

Claude Sonnet 4.5
& Always on & 100.0 & 59.16 & -- & 0.4038 \\
\bottomrule
\end{tabularx}
\end{table}

\begin{table}[t]
\centering
\caption{Comparison with existing escalation and uncertainty estimation methods, StepWise~\citep{wei2026step} and Progress Advantage~\citep{oh2026neglected}, with EvoCUA-8B on OSWorld.}
\label{tab:baseline_comparison}

\small
\renewcommand{\arraystretch}{1.08}
\setlength{\tabcolsep}{5pt}
\vspace{-0.5em}
\begin{tabularx}{\linewidth}{@{}lXrrr@{}}
\toprule
Backbone
& Method
& Esc. (\%)
& Success (\%) $\uparrow$
& Cost (\$) $\downarrow$ \\
\midrule

EvoCUA-8B
& -
& 0.0
& 42.37
& -- \\

& StepWise~\citep{wei2026step}
& 49.6
& 58.40
& 0.2824 \\

& Progress Advantage~\citep{oh2026neglected}
& 46.9
& 62.98
& 0.3533 \\

\rowcolor{ourscolor}
\cellcolor{white}
& \textbf{\method{}}
& 45.4
& \textbf{67.18}
& 0.3627 \\

\bottomrule
\end{tabularx}
\end{table}

\method{} consistently improves task success over standalone weak-agent execution across all four backbones (Table~\ref{tab:main_results}). At closely matched escalation rates, it also outperforms random task escalation by 6.55--11.94 percentage points, showing that the training-free uncertainty signal provides useful information for selective weak-to-strong escalation rather than merely increasing access to the strong executor. With Qwen3.5-35B-A3B~\citep{qwen35blog} as the backbone, \method{} achieves 73.66\% task success while escalating 64.1\% of tasks, compared with 72.90\% for always-on Opus~\citep{anthropic2026opus5}. Besides, its estimated inference cost is
\$0.4117 per task, 19.2\% lower than the \$0.5094 cost of always-on Opus. The benefit also persists at lower escalation rates:
\method{} improves Nemotron-3~\citep{deshmukh2026nemotron} from 45.04\% to 66.03\% success while escalating 43.1\% of tasks, and EvoCUA-8B~\citep{xue2026evocua} from 42.37\% to 67.18\% while escalating 45.4\%. This suggests that \method{} can substantially improve overall performance while limiting reliance on the strong executor, yielding a favorable accuracy--cost trade-off across diverse backbones.
The gains are also consistent across weak agents with substantially different standalone capabilities. 
Moreover, all four \method{} configurations outperform the always-on Sonnet, and three of them do so at lower estimated inference cost. That is, for most backbones, selective escalation to a frontier model is both more accurate and cheaper than running a mid-tier model on every task.

We also compare against StepWise~\citep{wei2026step} and Progress Advantage~\citep{oh2026neglected} in Table~\ref{tab:baseline_comparison}. We reproduce StepWise using its released monitor, and for Progress Advantage, we use EvoCUA-8B and Qwen3-VL-8B-Instruct as the final and reference policy, respectively, and use its turn-level mean aggregate advantage as negative uncertainty. Both methods use Opus 5 as the strong executor and the same cost accounting as \method{}. We see that \method{} outperforms other methods with a higher success rate at a comparable escalation rate.

\subsection{Uncertainty Estimation}
\label{sec:uncertainty_results}

\begin{table}[t]
\centering
\caption{AUROC of uncertainty estimators. For the probability-based column, we report the best-performing uncertainty statistic in each setting. Initial and runtime AUROC use task outcomes and Manager-assigned subgoal outcomes, respectively. 
}
\label{tab:uncertainty_quality}
\small
\setlength{\tabcolsep}{4pt}
\renewcommand{\arraystretch}{1.08}
\begin{tabular*}{\linewidth}{@{\extracolsep{\fill}}llrrrr@{}}
\toprule
Stage & Backbone & Probability-based & Euclidean & kNN & \textbf{Pooled LDA (Ours)} \\
\midrule
Initial & Nemotron-3 (30B-A3B) & 0.575 & 0.692 & 0.672 & \textbf{0.749} \\
 & Qwen3.8-27B & \textbf{0.671} & 0.541 & 0.572 & 0.600 \\
\addlinespace
Runtime & Nemotron-3 (30B-A3B) & 0.488 & 0.583 & 0.489 & \textbf{0.609} \\
 & Qwen3.8-27B & 0.540 & 0.564 & 0.522 & \textbf{0.575} \\
\bottomrule
\end{tabular*}
\end{table}

We compare output probability-based uncertainty estimates with three representation-based alternatives: Euclidean distance to class centroids, kNN over reference representations, and our pooled LDA using covariance-aware Mahalanobis distance.
The probability-based baseline considers several uncertainty metrics, including negative log-probability, chosen-token probability, predictive entropy, and nucleus-set size under different aggregation schemes~\citep{ICLR2024_11f5520d}, and reports the strongest one in each setting.
As shown in Table~\ref{tab:uncertainty_quality}, the effectiveness of each uncertainty estimator (i.e., how well it predicts task or sub-goal failure) varies across both backbones and execution stages.
At the initial query-based uncertainty estimation stage, pooled LDA performs best for Nemotron, whereas probability-based baseline is stronger for Qwen3.8-27B. At the runtime step-wise estimation stage, our pooled LDA score achieves the highest AUROC for both backbones, although the overall effectiveness remains moderate. 
Overall, these results suggest that representation-based uncertainty can provide a training-free signal for runtime monitoring, particularly highlighting the effectiveness of our LDA-based estimation, while its effectiveness remains dependent on the backbone and execution stage. Meanwhile, the uncertainty estimator is modular in \method{}, allowing alternative, better uncertainty signals to be substituted without changing the escalation workflow depending on available resources or data~\citep{park2026vauq,park2026hide}.

\subsection{Ablation Study}
\label{sec:ablation}

\begin{table}[t]
\centering
\caption{Component ablation of \method{} on Qwen3.5-35B-A3B.
Checkmarks indicate enabled components. Shadow includes
post-handoff oversight and checkpoint intervention.}
\label{tab:component_ablation}
\small
\renewcommand{\arraystretch}{1.15}
\setlength{\tabcolsep}{5pt}
\begin{tabular*}{\linewidth}{@{\extracolsep{\fill}}lcccrr@{}}
\toprule
Configuration
& Manager
& Escalation
& Shadow
& Success (\%) $\uparrow$
& Cost (\$) $\downarrow$ \\
\midrule
Weak agent alone
& $\times$ & $\times$ & $\times$
& 37.40 & -- \\

+ Manager 
& $\checkmark$ & $\times$ & $\times$
& 41.15 & 0.1372 \\

+ Escalation
& $\checkmark$ & $\checkmark$ & $\times$
& 67.94 & 0.4050 \\

\midrule
\textbf{\method{}}
& $\checkmark$ & $\checkmark$ & $\checkmark$
& \textbf{73.66} & 0.4117 \\
\bottomrule
\end{tabular*}
\end{table}

\paragraph{Component-wise ablation.}
As shown in Table~\ref{tab:component_ablation}, the largest gain comes from enabling uncertainty-guided escalation: task success increases from 41.15\% with the Manager to 67.94\%, at a higher inference cost. Adding Shadow monitoring and checkpoint intervention further improves success to 73.66\% with only a small additional cost, revealing an isolated effectiveness of the post-handoff intervention. Additional backbone ablations are provided in Appendix~\ref{app:component_ablations}.

\paragraph{What does Shadow oversight detect?}
We further examine the tasks on which post-handoff monitoring fires in Table~\ref{tab:shadow_paired}.
Shadow observer triggers concentrate on particularly difficult trajectories: always-on Opus succeeds on only 51.6\% of triggered tasks, compared with 74.5\% of escalated tasks without a trigger. On the triggered subset, the Shadow-enabled policy achieves 64.5\% success, compared with 48.4\% without Shadow. These suggest that the Shadow observer identifies cases where difficulty persists even after escalation, as we expected. 
\begin{table}[t]
\centering
\caption{Subset analysis of post-handoff Shadow monitoring with Qwen3.5-35B-A3B. \emph{Both escalated} contains tasks handed off by both Shadow and No-Shadow policies. \emph{Shadow triggered} contains tasks with at least one checkpoint trigger, while \emph{Escalated, not triggered} contains escalated tasks without a trigger. \emph{Opus only} reports always-on Opus accuracy on the same task IDs.}
\label{tab:shadow_paired}
\small
\renewcommand{\arraystretch}{1.10}
\setlength{\tabcolsep}{4.5pt}

\begin{tabular*}{\linewidth}{@{\extracolsep{\fill}}lrrrr@{}}
\toprule
Subset
& Tasks
& Shadow (\%)
& No Shadow (\%)
& Opus only (\%) \\
\midrule
Both escalated
& 131 & 69.5 & 67.9 & 64.9 \\

Shadow triggered
& 31 & \textbf{64.5} & 48.4 & 51.6 \\

Escalated, not triggered
& 137 & 74.5 & 69.3 & 74.5 \\
\bottomrule
\end{tabular*}
\end{table}

\subsection{Sensitivity Analysis}
\label{sec:sensitivity}

\paragraph{Which tasks to escalate.}
We first examine whether uncertainty helps identify which tasks should receive strong-model assistance. Figure~\ref{fig:escalation_rate} varies the escalation budget and compares \method{} with random task-level routing at matched escalation rates. 
The random baseline is evaluated at the same operating points, shown by the hollow markers; its near-linear trend follows from uniformly allocating a fixed fraction of tasks to the strong agent. Across varying escalation rates, \method{} consistently achieves higher task success, indicating that its gains come from selectively allocating strong-model computation to more informative tasks rather than simply increasing strong-model usage.

\begin{figure}[t]
    \centering
    \begin{minipage}[t]{0.48\linewidth}
        \centering
        \includegraphics[width=\linewidth]{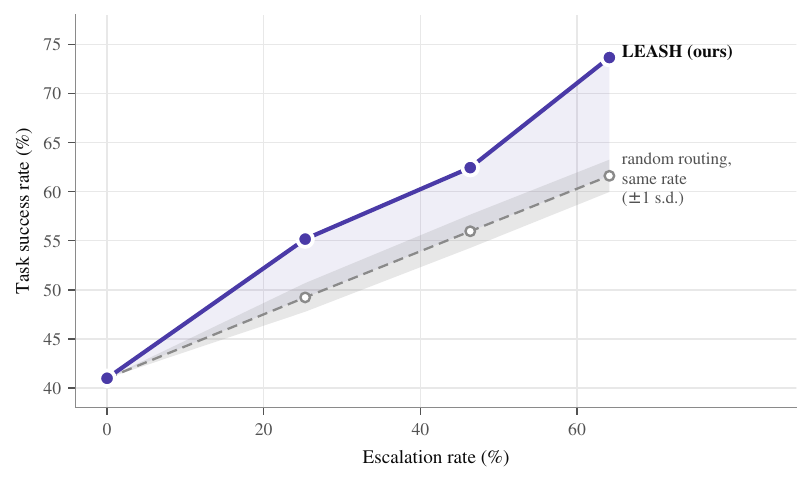}
        \caption{Sensitivity to escalation rate on Qwen3.5-35B-A3B. Hollow markers denote the random-routing operating points, each estimated over 4,000 random task allocations.}
        \label{fig:escalation_rate}
    \end{minipage}
    \hfill
    \begin{minipage}[t]{0.48\linewidth}
        \centering
        \includegraphics[width=\linewidth]{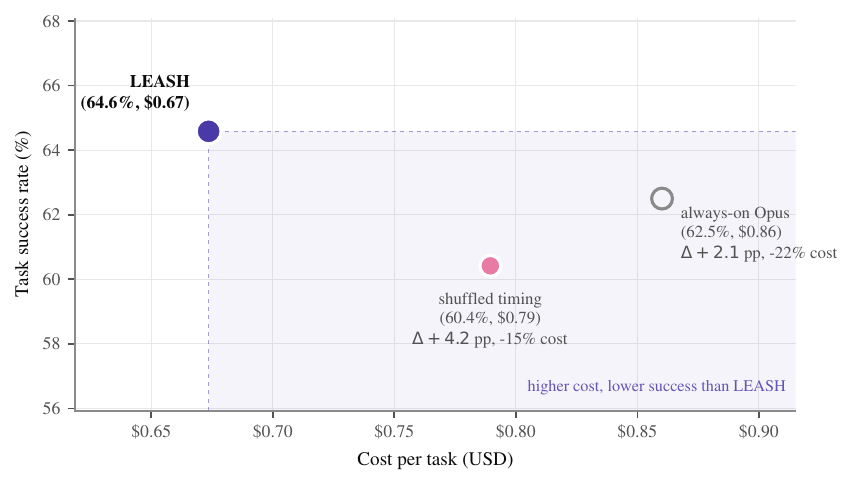}
        \caption{Effect of escalation timing on escalation-prone tasks.
\method{} achieves higher success at lower estimated cost than shuffled runtime escalation and always-on Opus.}
        \label{fig:timing}
    \end{minipage}
\end{figure}
\paragraph{When to escalate.}
We next examine whether the timing of handoff matters once a task is likely to require escalation. We focus on an escalation-prone subset consisting of the tasks most frequently escalated by \method{} and compare \method{} with a shuffled-timing baseline. As shown in Figure~\ref{fig:timing}, \method{} achieves 64.6\% task success at \$0.67 per task, compared with 60.4\% at \$0.79 under shuffled timing. This 4.2-point improvement, together with 15\% lower cost, suggests that effective weak-to-strong escalation depends not only on selecting the right tasks, but also on transferring control at the right point during execution.

\paragraph{Reference bank size.}
We examine the effect of reference bank size on initial-stage uncertainty estimation. As shown in Table~\ref{tab:bank_sensitivity}, increasing the bank size from 40 to 99 improves AUROC from 0.701 to 0.770 for Qwen3.5-35B-A3B and from 0.633 to 0.696 for EvoCUA-8B, highlighting its compatibility for recursively self-improving system~\citep{yin-etal-2025-godel,ouyang2026reasoningbank,lee2026recursive}.

\begin{table}[t]
\centering
\caption{Sensitivity to reference bank size. Initial-stage AUROC is reported as mean $\pm$ standard deviation across subsamples.}
\label{tab:bank_sensitivity}
\small
\setlength{\tabcolsep}{10pt}
\begin{tabular}{lcc}
\toprule
Bank size & Qwen3.5-35B-A3B & EvoCUA-8B \\
\midrule
40 & $0.701\pm0.100$ & $0.633\pm0.099$ \\
60 & $0.744\pm0.060$ & $0.654\pm0.061$ \\
80 & $0.749\pm0.036$ & $0.681\pm0.049$ \\
99 & $0.770\pm0.021$ & $0.696\pm0.022$ \\
\bottomrule
\end{tabular}
\end{table}

\section{Theoretical Analysis}
\label{sec:initial_theory}

Section~\ref{sec:method} introduced the distance-based uncertainty score in Eq.~(\ref{eq:uncertainty_score}), and Section~\ref{sec:experiments} showed that escalating on that score outperforms random escalation at matched rates. In this section, we now explain the empirical success of our uncertainty-based escalation through a simple probabilistic analysis. 
To keep the analysis simple, we consider a single decision per task: each task is either handed to the strong agent before execution begins ($t_{\mathrm{esc}}=0$) or never handed off ($t_{\mathrm{esc}}=T$). This isolates the initial gate $g_0$ of \method{} (Section~\ref{sec:escalation}), with no runtime escalation afterward. We also assume the budget in Eq.~(\ref{eq:objective}) is used exactly, so a fixed fraction $p$ of tasks is escalated, which is what our quantile thresholds target.
Let $z$ be the standardized query representation, $Y\in\{\pass,\fail\}$ the outcome if the weak agent completes the task without escalation, and $q_{\mathrm{w}}(z)=\Pr(Y=\pass\mid z)$ and $q_{\mathrm{s}}(z)$ the success probabilities when the weak agent and the strong agent, respectively, completes the task. We analyze $u(z)$, the population version of our initial score $u_q$ in Eq.~(\ref{eq:uncertainty_score}), computed with the true class means and covariance instead of their estimates.

\begin{proposition}[Optimality of \method{}'s initial gate at a fixed escalation rate]\label{prop:lr}
Suppose $z\mid Y=y\sim\mathcal{N}(\mu_y,\Sigma)$ with shared covariance $\Sigma\succ0$ and distinct means, and both outcomes occur with positive probability. Let the initial gate escalate a task when $u(z)$ exceeds the threshold that escalates a fraction $p$ of tasks. Among all escalation rules based on $z$ with the same rate $p$:
\begin{enumerate}[label=(\roman*),leftmargin=*,nosep]
\item the gate hands off the largest fraction of the tasks the weak agent would fail, i.e., it maximizes $\Pr(\text{escalated}\mid Y=\fail)$;
\item if the escalation gain $\Delta(z)=q_{\mathrm{s}}(z)-q_{\mathrm{w}}(z)$ is a nondecreasing function of $u(z)$, and when $q_{\mathrm{s}}$ is constant, the gate also maximizes task success rate.
\end{enumerate}
\end{proposition}

\paragraph{Proof sketch.} Under this model, our score equals (up to an additive constant) twice the log-odds that the weak agent fails (Lemma~\ref{lem:lda} in Appendix~\ref{app:score_theory}), which recalls the classical linear discriminant analysis identity~\citep{fisher1936use}. Escalating the highest-scoring tasks thus escalates the tasks most likely to fail, and a Neyman--Pearson exchange argument~\citep{neyman1933problem} shows that no other rule at the same rate catches more failures; (ii) follows by the same argument applied to $\Delta$.

\paragraph{Implication.} Proposition~\ref{prop:lr}~(i) says that, at any escalation budget, no escalation rule based on the query representation $z$ captures more of the weak agent's failures than the initial gate of \method{}. For example, random escalation with budget $p$ hands off only a fraction $p$ of these failures, whereas our gate hands off strictly more. Moreover, (ii) shows that our gate spends the strong-agent budget where it pays off most~\citep{NEURIPS2023_1f09e1ee}, so catching more of the weak agent's failures directly turns into higher overall task success. Meanwhile, the analysis covers only the initial gate; analyzing the runtime gate would require an optimal-stopping formulation, which we leave to future work.

\section{Conclusion}
\label{sec:conclusion}
We introduced \emph{weak to strong escalation} as a general inference paradigm for adaptive agent systems, in which inexpensive weak models handle routine execution and stronger models are invoked selectively when additional capability is needed. Under this paradigm, we proposed \method{}, a training-free, uncertainty-aware framework that supports initial routing, runtime escalation, and continued post-handoff oversight. Our results suggest that a shared latent uncertainty signal can support both selective escalation and post-handoff oversight, providing a practical mechanism for dynamically allocating model capability while retaining continued control over strong-agent execution.

\section*{Limitations}
\label{sec:limitations}

Our current evaluation is limited to computer-use agents on OSWorld, so the generality of weak-to-strong escalation beyond this domain remains to be validated. In addition, our uncertainty score estimates weak-agent failure risk rather than the expected benefit of escalation; Proposition~\ref{prop:lr} requires a condition that the escalation gain be nondecreasing in the score, which may fail when both weak and strong agents struggle on the same task. 
Finally, post-handoff Shadow monitoring operates under a distribution shift from weak-agent trajectories, and its intervention effect is not causally isolated in our current evaluation.

\newpage

\bibliography{main}

\newpage
\appendix

\section{Method and Formulation Details}
\label{app:method_details}

\subsection{Routing Details, Analyses and Proofs}
\label{app:score_theory}

\paragraph{Setup and notation.} Section~\ref{sec:initial_theory} analyzes a single routing decision made before execution, i.e., Eq.~(\ref{eq:objective}) restricted to $t_{\mathrm{esc}}\in\{0,T\}$. We keep both agents fixed and compare only routing rules. Let $q_{\mathrm{w}}(z)=\Pr(Y=\pass\mid z)$ be the success probability if the weak agent completes the task without escalation, and $q_{\mathrm{s}}(z)$ the success probability if the strong agent completes it. A routing rule $\phi(z)\in[0,1]$ escalates a task with probability $\phi(z)$, using only $z$, i.e., without knowing whether the task would actually succeed. Defining $\Delta(z)=q_{\mathrm{s}}(z)-q_{\mathrm{w}}(z)$ for the success gain from escalation, the overall success rate decomposes as
\begin{equation}
J(\phi)=\mathbb{E}\bigl[(1-\phi(z))q_{\mathrm{w}}(z)+\phi(z)q_{\mathrm{s}}(z)\bigr]
=\mathbb{E}[q_{\mathrm{w}}(z)]+\mathbb{E}[\phi(z)\Delta(z)].
\label{eq:routing_gain}
\end{equation}
For the proofs, we will use the threshold-exchange argument underlying the Neyman--Pearson lemma~\citep{neyman1933problem}, with a constraint on the marginal escalation rate rather than a class-conditional false-positive rate.

\paragraph{Relative distance and the fitted score.} For either channel, expanding Eq.~(\ref{eq:uncertainty_score}) gives
\begin{equation}
\begin{aligned}
u_c(h) ={}& 2\,(\mu_{c,\fail}-\mu_{c,\pass})^{\top}\Sigma_{c,\lambda}^{-1}z_c(h) \\
&+\mu_{c,\pass}^{\top}\Sigma_{c,\lambda}^{-1}\mu_{c,\pass}
-\mu_{c,\fail}^{\top}\Sigma_{c,\lambda}^{-1}\mu_{c,\fail}.
\end{aligned}
\label{eq:linear_form}
\end{equation}
Since the shared quadratic term cancels, $u_c$ measures displacement along $\Sigma_{c,\lambda}^{-1}(\mu_{c,\fail}-\mu_{c,\pass})$. Note that this identity holds for the fitted score without any distributional assumption, and the Gaussian model is needed only for the population likelihood-ratio interpretation below.

\begin{lemma}[Distance score as failure log-odds]\label{lem:lda}
Under the Gaussian model of Proposition~\ref{prop:lr}, let $\pi_y=\Pr(Y=y)>0$ be the class priors and $d_y^2(z)=(z-\mu_y)^\top\Sigma^{-1}(z-\mu_y)$, so that $u(z)=d_{\pass}^2(z)-d_{\fail}^2(z)$. Then the weak agent's failure risk $\eta(z)=\Pr(Y=\fail\mid z)$ satisfies
\begin{equation}
\log\frac{\eta(z)}{1-\eta(z)}
=\tfrac12u(z)+\log\frac{\pi_{\fail}}{\pi_{\pass}},
\label{eq:failure_log_odds}
\end{equation}
so $\eta$ is strictly increasing in $u$.
\end{lemma}

\begin{proof}[Proof of Lemma~\ref{lem:lda}]
Let $f_y$ denote the class-conditional Gaussian density. The shared covariance cancels the density normalizer terms, resulting in
\begin{equation*}
\log\frac{f_{\fail}(z)}{f_{\pass}(z)}
=-\tfrac12d_{\fail}^2(z)+\tfrac12d_{\pass}^2(z)
=\tfrac12u(z).
\end{equation*}
By Bayes' rule, $\log\{\eta/(1-\eta)\}=\log(f_{\fail}/f_{\pass})+\log(\pi_{\fail}/\pi_{\pass})$, which gives Eq.~(\ref{eq:failure_log_odds}). The log-odds are strictly increasing in $\eta$, so $\eta$ is a strictly increasing function of $u$.
\end{proof}

\Needspace{10\baselineskip}
\begin{proof}[Proof of Proposition~\ref{prop:lr}]
Denote $\delta=\mu_{\fail}-\mu_{\pass}\ne0$. By the expansion in Eq.~(\ref{eq:linear_form}), $u(z)=2\delta^\top\Sigma^{-1}z+b$ for a constant $b$, so within each class $u$ is Gaussian with variance $4\delta^\top\Sigma^{-1}\delta>0$. The distribution of $u$ is therefore a mixture of two Gaussians with positive variance, which has a positive density on $\mathbb{R}$, so a unique $\tau_p$ with $\Pr(u>\tau_p)=p$ exists. The initial gate at rate $p$ is $\phi^\star(z)=\mathbb{I}\{u(z)>\tau_p\}$.

\emph{(i)} Let $\eta_p$ be the failure risk at $u=\tau_p$, and let $\phi$ be any rule with $\mathbb{E}[\phi]=p$. Since $\phi$ uses only $z$, a task with representation $z$ is escalated with probability $\phi(z)$ regardless of whether it would fail, so $\Pr(\text{escalated},Y=\fail)=\mathbb{E}[\phi\eta]$. Because both rules escalate the same fraction of tasks, $\mathbb{E}[\phi^\star-\phi]=0$, and we may subtract the constant $\eta_p$:
\begin{equation*}
\mathbb{E}[\phi^\star\eta]-\mathbb{E}[\phi\eta]
=\mathbb{E}\bigl[(\phi^\star-\phi)(\eta-\eta_p)\bigr]\ge0.
\end{equation*}
By Lemma~\ref{lem:lda}, $\eta\ge\eta_p$ wherever $\phi^\star=1$, and $\eta\le\eta_p$ wherever $\phi^\star=0$. 
For random escalation, $\phi\equiv p$, the integrand $(\phi^\star-p)(\eta-\eta_p)$ is strictly positive whenever $u\ne\tau_p$, because $\eta$ is strictly increasing in $u$. Therefore, the inequality is strict, and while random escalation hands off exactly a fraction $p$ of the failures, $\phi^\star$ hands off strictly more.
Dividing by $\pi_{\fail}$, $\phi^\star$ also maximizes $\Pr(\text{escalated}\mid Y=\fail)$.

\emph{(ii)} Since $Y$ is the weak agent's outcome, $q_{\mathrm{w}}(z)=1-\eta(z)$, so $\Delta(z)=q_{\mathrm{s}}(z)-1+\eta(z)$. Suppose $\Delta(z)=g(u(z))$ for a nondecreasing $g$, and let $\Delta_p=g(\tau_p)$. By Eq.~(\ref{eq:routing_gain}) and the same subtraction,
\begin{equation*}
J(\phi^\star)-J(\phi)
=\mathbb{E}\bigl[(\phi^\star-\phi)\Delta\bigr]
=\mathbb{E}\bigl[(\phi^\star-\phi)(\Delta-\Delta_p)\bigr]\ge0,
\end{equation*}
since $\Delta\ge\Delta_p$ wherever $\phi^\star=1$ and $\Delta\le\Delta_p$ wherever $\phi^\star=0$. If $q_{\mathrm{s}}$ is constant, $g$ is strictly increasing, so the condition holds.
\end{proof}

\subsection{Representations and Estimator Fitting}
\label{app:method_impl}

\paragraph{Layer and token.} We use the last transformer layer of each weak backbone (Table~\ref{tab:layers}). For the runtime channel, $h_t$ is the hidden state at the last generated token of the action; we apply no pooling within a step and no temporal window across steps, so every step contributes one representation. For the initial channel, we reissue the task's first request with a one-token generation limit and read the hidden state at the final prompt position, obtained during prefill.

\begin{table}[h]
\centering
\caption{Hidden layer used for each weak backbone (0-indexed layer / number of layers).}
\label{tab:layers}
\small
\begin{tabular}{lc}
\toprule
Backbone & Layer \\
\midrule
EvoCUA-8B & 35 / 36 \\
Qwen3.5-35B-A3B & 39 / 40 \\
Nemotron-3-Nano-Omni (30B-A3B) & 51 / 52 \\
Qwen3.8-27B & 63 / 64 \\
\bottomrule
\end{tabular}
\end{table}

\paragraph{Standardization and covariance.} Standardization uses the feature-wise mean and standard deviation of the channel's bank, pooled across both classes. The class centroids are computed in the standardized space, and $\widehat{\Sigma}_c$ is the pooled within-class scatter normalized by $n-2$, where $n$ is the number of representations in the bank. 

\paragraph{Selecting the ridge parameter.} We select $\lambda_c$ from a grid spanning $10^{-2}$ to $10^{3}$ by task-grouped 10-fold cross-validated AUROC, separately for each backbone and channel. Grouping by task is essential: steps from one trajectory share much of their context and, at runtime, their segment label. In a preliminary check, a split that ignored task membership produced a near-perfect AUROC of 0.997, which reflected leakage across steps of the same trajectory rather than genuine discrimination.

\subsection{Reference Banks and Threshold Calibration}
\label{app:method_calib}

\paragraph{Reference banks.}
For each backbone, we construct a small held-out set of Manager-guided trajectories collected with escalation disabled and disjoint from the evaluation tasks. From this set, we construct the query and action representation banks: initial representations are labeled by task-level outcomes, while runtime representations inherit the outcomes of their corresponding subgoal segments.

\paragraph{Held-out scoring.}
We use the same held-out reference set for both representation-bank construction and threshold calibration. Specifically, we perform task-grouped 10-fold cross-validation, keeping all steps from the same task in one fold. For each fold, an uncertainty estimator is fitted on the other nine folds and used only to score the held-out fold, yielding one out-of-fold score for every reference
task. The fold-specific estimators are used only for calibration and are discarded afterwards.

We then set the escalation thresholds as empirical quantiles of these out-of-fold scores. After calibration, we fit one final estimator per uncertainty channel on all tasks in the same reference set, keep the calibrated thresholds fixed, and use these final estimators to score the evaluation tasks. No evaluation data are used for either estimator fitting or threshold calibration.

\paragraph{Targets and thresholds.} All four backbones use the same targets, $p_{q}=0.25$ and $p_{\mathrm{esc}}=0.50$, so that $p_{a}=(p_{\mathrm{esc}}-p_{q})/(1-p_{q})=1/3$. The thresholds $\tau_{q}$ and $\tau_{a}$ are exact order statistics of the held-out scores in Eq.~(\ref{eq:thresholds}). Because the reference set is finite, the achievable rates are discrete, with a resolution of one reference task.

\subsection{Post-Handoff Oversight}
\label{app:method_shadow}

\paragraph{Probe schedule and context.} Shadow probes run at the first strong-agent step after handoff and at every third step thereafter. At every strong-agent step, including steps without a probe, the executed action is appended to the weak model's history in that model's own action format, so that each probe conditions on the actual trajectory rather than on the pre-handoff context. A probe calls the weak agent's own action-generation routine, with its usual system prompt, tool schema, and image history, and does not invoke the Manager. The candidate action is logged for analysis but never executed or shown to the strong agent.

\paragraph{Shadow threshold.} $\tau_{\mathrm{sh}}$ is the 90th percentile of the pooled per-step runtime scores over the 99-task reference pool, computed separately for each backbone. It is derived directly from the reference pool rather than from the escalation thresholds. For EvoCUA-8B, the in-sample value ($+1.7309$) is close to the corresponding out-of-fold value ($+1.7106$).

\paragraph{Checkpoint budget.} At most six checkpoints can fire in an episode. In addition, once two checkpoints have changed the strong agent's action, further firings are suppressed; checkpoints that leave the action unchanged do not count toward this second limit. Because probes occur every three steps, consecutive checkpoints are at least three steps apart. 

\paragraph{Checkpoint mechanics.} When a checkpoint fires, the strong agent's conversation is truncated to its length at the start of the current step, which removes the pending action; the neutral checkpoint message in Figure~\ref{fig:handoff_checkpoint_prompts} is appended; and the strong agent re-decides the step with a raised reasoning-effort setting. The weak model's candidate action is deliberately excluded from the checkpoint message, so the strong agent is never steered toward a specific action proposed by the weaker model.

\section{Additional Experimental Results}
\label{app:additional_experiments}

\subsection{Experimental Details}
\label{app:experimental_details}

\paragraph{Evaluation setting.}
We evaluate on OSWorld with screenshot observations and a shared budget of 50 actions per task. The lightweight backbones are EvoCUA-8B~\citep{xue2026evocua}, Qwen3.5-35B-A3B~\citep{qwen35blog}, Nemotron-3-Nano-Omni (30B-A3B)~\citep{deshmukh2026nemotron}, and Qwen3.8-27B~\citep{qwen38}. Claude Opus 5~\citep{anthropic2026opus5} serves as both the Manager and strong executor, and Claude Sonnet 4.5~\citep{anthropic2025sonnet45} is an additional always-on reference. The always-on references run the strong model alone, without the Manager, so they measure strong-model capability without the Manager or escalation framework.

\paragraph{Configurations and evaluation sets.}
\emph{Manager} adds subgoal planning and verification; \emph{Escalation} further enables initial and runtime handoff; and \method{} additionally enables post-handoff Shadow monitoring and corresponding intervention. These configurations are cumulative. Shared results below follow Tables~\ref{tab:main_results} and~\ref{tab:component_ablation}. The evaluation task pool contains 262 tasks. 

\paragraph{Reference banks and calibration.}
We use a held-out reference set that is disjoint from the 262 evaluation tasks. We obtain task-grouped 10-fold out-of-fold scores, keeping all observations from a task in one fold; preprocessing and estimator fitting use only the training folds. Initial scores use final task outcomes and runtime scores use Manager-assigned subgoal outcomes. Thresholds are order statistics of the held-out scores targeting a 25\% initial and 50\% overall escalation frequency, with runtime calibration based on per-task maximum scores. Refitting slightly changes the score scale, and differences between the reference and evaluation sets can further affect the realized escalation rate; hence the observed rates (43--64\%) need not exactly match the calibration targets.

\paragraph{Post-handoff execution.}
After handoff, Shadow probes query the lightweight model every three strong-agent steps. Probe actions are not executed; their hidden representations provide the monitoring scores. Higher scores indicate greater uncertainty and can trigger an intervention that asks the strong agent to reassess its current strategy while retaining control.

\paragraph{Cost estimation.}
We estimate inference cost from recorded token usage using OpenRouter reference prices, accounting for cached tokens where applicable. Total cost includes lightweight-agent, Manager, strong-executor, and monitoring calls. For locally served models, we report reference inference-cost estimates rather than measured hardware costs.

\subsection{Uncertainty Estimation and Reference-Bank Sensitivity}
\label{app:uncertainty_details}

\paragraph{Estimator comparison.}
We compare pooled LDA with Euclidean distance, kNN, and probability-based baselines. The probability-based baseline considers negative log-probability, chosen-token probability, predictive entropy, and nucleus-set size under several aggregation schemes; we report the strongest statistic in each setting. Performance varies across backbones and execution stages, with pooled LDA providing the strongest overall discrimination in most settings. Runtime AUROC reflects discrimination of Manager-assigned subgoal-segment outcomes rather than individual action correctness.

\paragraph{Reference-bank size.}
Table~\ref{tab:bank_sensitivity} reports initial-stage AUROC as the bank is subsampled, it also shows that estimates become more stable as the bank grows: the standard deviation decreases from roughly 0.10 at 40 tasks to about 0.02 at 99 tasks. Because thresholds are quantiles of this same pool, a small bank also gives coarse control of the realized escalation rate.

\subsection{Additional Component Ablations}
\label{app:component_ablations}

\paragraph{Initial and runtime escalation.}
Table~\ref{tab:app_gate_ablation} compares different gate configurations for Qwen3.5-35B-A3B. Runtime-only escalation achieves higher success than initial-only escalation, at higher cost, indicating that runtime monitoring provides the stronger escalation signal. The complete framework attains the best overall performance, but it also enables Shadow oversight. Its counterpart without Shadow is the \emph{Escalation} configuration in Table~\ref{tab:component_ablation} (67.94\% at \$0.4050), which falls below runtime-only escalation, so this ablation does not show a benefit from adding the initial gate to the runtime gate in this setting.

\begin{table}[t]
\centering
\caption{Gate configurations for Qwen3.5-35B-A3B. Each configuration retains the Manager; the first two disable post-handoff Shadow monitoring and checkpoint intervention.}
\label{tab:app_gate_ablation}
\small
\renewcommand{\arraystretch}{1.10}
\setlength{\tabcolsep}{6pt}
\begin{tabular*}{\linewidth}{@{\extracolsep{\fill}}lrr@{}}
\toprule
Gate configuration & Success (\%) $\uparrow$ & Cost (\$) $\downarrow$ \\
\midrule
Initial only & 59.92 & 0.3122 \\
Runtime only & 71.65 & 0.4169 \\
\method{} & \textbf{73.66} & 0.4117 \\
\bottomrule
\end{tabular*}
\end{table}

\subsection{Inference Cost}
\label{app:cost_breakdown}

Runtime uncertainty scoring reuses hidden states from the weak agent's action generation and therefore requires no additional model call. Initial and Shadow probes require lightweight-model inference, while interventions may incur additional strong-model computation. All such costs are already included in the total workflow cost reported in Table~\ref{tab:main_results}.


\section{Prompt Templates}
\label{app:prompts}

Figures~\ref{fig:manager_prompts} and~\ref{fig:handoff_checkpoint_prompts} show the core prompts used for planning, verification, handoff, and checkpoint intervention. The planning prompt is excerpted, with omissions marked by \texttt{[...]}; braces denote runtime placeholders in the handoff template. Escalation gates use statistical scores and require no prompt. Shadow candidate actions are neither executed nor shown to the strong agent.

\begin{figure*}[!t]
\centering
\begin{tcolorbox}[
  title={Manager: task decomposition (system prompt excerpt)},
  colback=gray!3,colframe=black!55,
  fonttitle=\small\bfseries,boxrule=0.5pt,arc=1mm,
  left=2mm,right=2mm,top=1mm,bottom=1mm]
\begin{Verbatim}[fontsize=\footnotesize,breaklines=true,breakanywhere=true]
You are a planning agent for a SINGLE computer-use agent operating ONE Ubuntu desktop. You are a regular LLM — you cannot click or type yourself; a downstream agent executes pyautogui actions one step at a time. Your job is to take one desktop task and decompose it into a SHORT, ORDERED list of subgoals the agent completes sequentially, each with a concrete done-criterion that a verifier can confirm from a screenshot.

[...]

## Output format
Respond with VALID JSON only (no markdown code fences, no extra prose). Schema:
{
  "task_analysis": "one sentence on the task and how you split it",
  "subgoals": [
    {"id": "subgoal_1", "description": "what to accomplish", "done": "observable on-screen condition that confirms completion"},
    {"id": "subgoal_2", "description": "...", "done": "..."}
  ]
}
Rules: ids are unique and sequential (subgoal_1, subgoal_2, ...); 1–6 subgoals is typical; the subgoals array must never be empty.
\end{Verbatim}
\end{tcolorbox}

\begin{tcolorbox}[
  title={Manager: strict verification (appended system instruction)},
  colback=gray!3,colframe=black!55,
  fonttitle=\small\bfseries,boxrule=0.5pt,arc=1mm,
  left=2mm,right=2mm,top=1mm,bottom=1mm]
\begin{Verbatim}[fontsize=\footnotesize,breaklines=true,breakanywhere=true]
## Strict verification (IMPORTANT)
When you choose `advance`, you MUST positively confirm from the CURRENT screenshot that the focus sub-goal's completion criterion is CONCRETELY and FULLY satisfied by the actual end-state — not merely that the relevant application, window, file, menu, or element is present, open, or focused. If you cannot positively verify the required end-state from what is visibly on screen, do NOT advance — choose `no_change` or `revise` instead. Do not assume an action succeeded just because it was attempted.
\end{Verbatim}
\end{tcolorbox}
\caption{Manager prompt excerpts. The strict verification instruction is appended to the replanning system prompt; it is not a standalone verifier prompt.}
\label{fig:manager_prompts}
\end{figure*}

\begin{figure*}[!t]
\centering
\begin{tcolorbox}[
  title={Strong-agent handoff (message template)},
  colback=gray!3,colframe=black!55,
  fonttitle=\small\bfseries,boxrule=0.5pt,arc=1mm,
  left=2mm,right=2mm,top=1mm,bottom=1mm]
\begin{Verbatim}[fontsize=\footnotesize,breaklines=true,breakanywhere=true]
### User (for each retained step):
Earlier step {s} (previous agent's attempt):
[Step screenshot]

### Assistant:
Thought: {reason[:200]}
Action: {nl[:120]}

### User (handoff):
[Current screenshot]
{instruction}

The attempts above were made by a weaker agent and did not complete the task. You are taking over. Work from the CURRENT screenshot.
\end{Verbatim}
\end{tcolorbox}

\begin{tcolorbox}[
  title={Post-handoff checkpoint (intervention message)},
  colback=gray!3,colframe=black!55,
  fonttitle=\small\bfseries,boxrule=0.5pt,arc=1mm,
  left=2mm,right=2mm,top=1mm,bottom=1mm]
\begin{Verbatim}[fontsize=\footnotesize,breaklines=true,breakanywhere=true]
CHECKPOINT. Pause before acting.
Answer these three questions in at most 60 words total:
  1. What is the current state of the task, and what have your recent actions actually changed?
  2. Can the approach you are currently using complete this task, or has it stalled?
  3. If it can, say so and continue with it. If it cannot, say what you will do differently.
Then take your next action.
\end{Verbatim}
\end{tcolorbox}
\caption{Handoff and checkpoint prompts. Handoff retains up to six recent weak-agent steps by default; bracketed screenshot markers indicate image inputs, not literal prompt text. The neutral checkpoint requests reassessment without asserting that execution has stalled and explicitly permits continuing the current approach.}
\label{fig:handoff_checkpoint_prompts}
\end{figure*}

\paragraph{Execution context.}
The inherited OSWorld system prompt mentions shell and editor tools, although this harness exposes only the \texttt{computer} tool. A system suffix corrects the sudo password for the Docker environment. The diagnostic state query is logging-only, and the alternative \texttt{think\_reflect} prompt is not used by the hidden-state checkpoint.

\section{Case Study}
\label{app:case_study}

\paragraph{Runtime escalation case studies.}
Figure~\ref{fig:runtime_case_1}--\ref{fig:runtime_case_3} present three representative examples of runtime escalation. Across all three cases, the weak agent remains on an unproductive GUI-based trajectory before the runtime gate triggers handoff. After takeover, the strong agent changes the execution strategy rather than simply continuing the weak agent's actions~\citep{song2026coact}.

In the first case, the weak agent repeatedly interacts with an unresponsive macro interface while attempting to standardize text-box fonts. \emph{LEASH} hands off at step 11, after which the strong agent switches from the macro GUI to a terminal-based \texttt{python-pptx} solution and completes the task in 27 actions; the Manager-only run instead exhausts the 50-action budget.

In the second case, the weak agent repeatedly opens, inspects, and closes candidate photos. Handoff occurs at step 8, and the strong agent replaces this loop with keyboard navigation and multi-selection, successfully creating the requested archive in 27 actions. The corresponding Manager-only run creates an incorrect archive and receives a score of zero.

In the third case, the weak agent spends 23 steps browsing and searching without installing the requested LibreOffice extension. After handoff at step 24, the strong agent switches to a terminal-based installation strategy and successfully installs LanguageTool, whereas the Manager-only run exhausts the action budget without completing the installation.

Together, these examples illustrate that runtime escalation can identify stalled or inefficient execution trajectories and enable a strong agent to adopt qualitatively different strategies rather than merely continuing the weak agent's current path.

\begin{figure*}[t]
    \centering
    \includegraphics[page=1,width=0.9\textwidth]
    {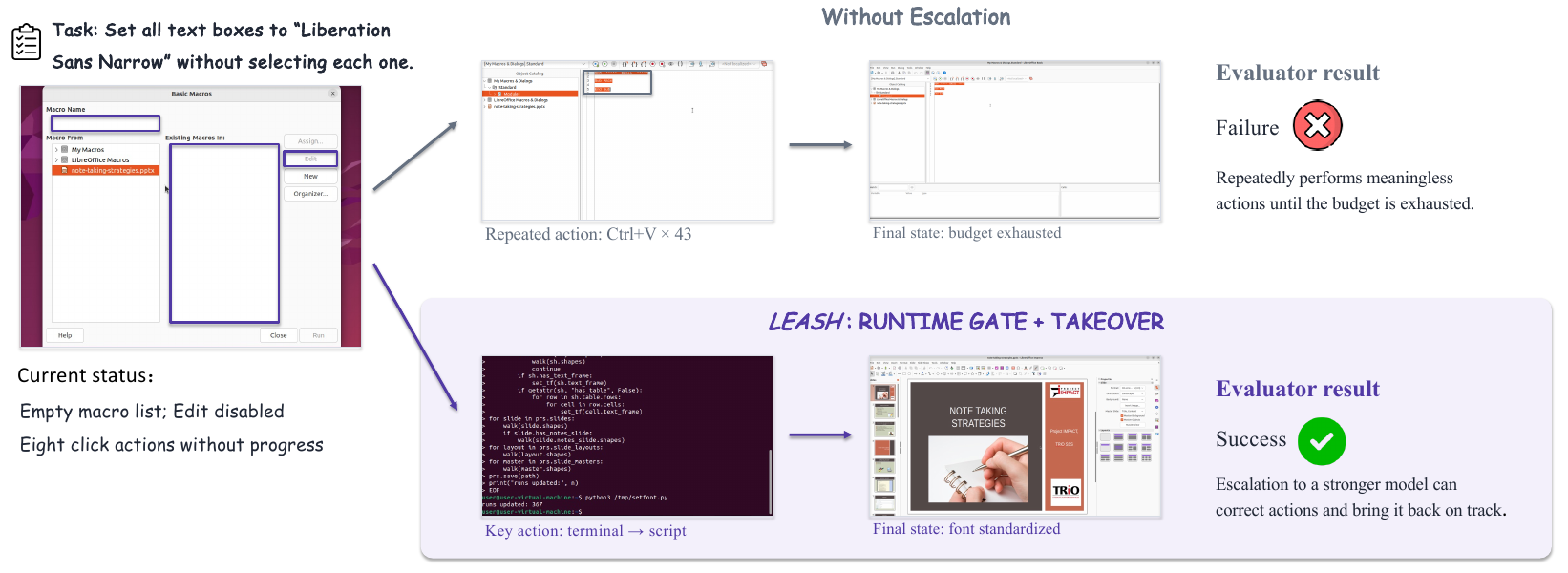}
    \caption{\textbf{Runtime escalation case study I: font standardization.} The weak agent becomes stuck in repeated macro-interface interactions. After handoff at step 11, the strong agent switches to a terminal-based \texttt{python-pptx} strategy and successfully completes the task.}
    \label{fig:runtime_case_1}
\end{figure*}

\begin{figure*}[t]
    \centering
    \includegraphics[page=2,width=0.9\textwidth]
    {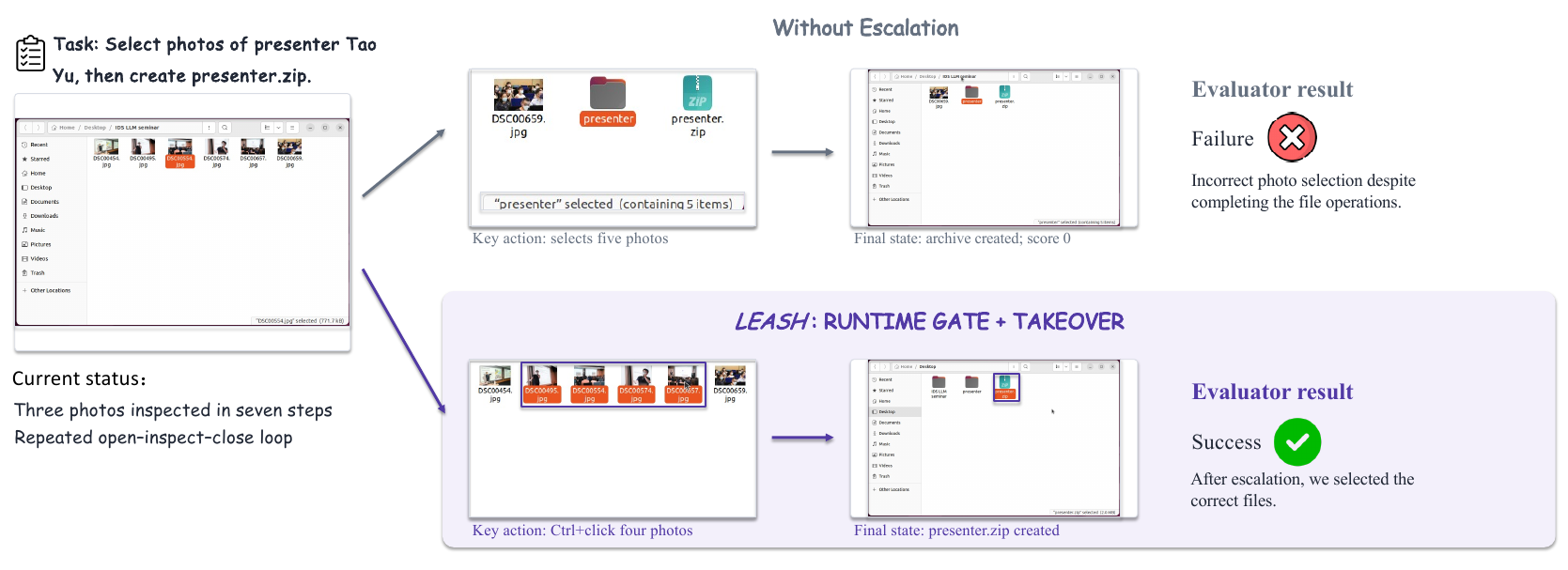}
    \caption{\textbf{Runtime escalation case study II: photo selection.} The weak agent enters a repeated open--inspect--close loop. After handoff at step 8, the strong agent changes to keyboard navigation and multi-selection and successfully creates the requested archive.}
    \label{fig:runtime_case_2}
\end{figure*}

\begin{figure*}[t]
    \centering
    \includegraphics[page=3,width=0.9\textwidth]
    {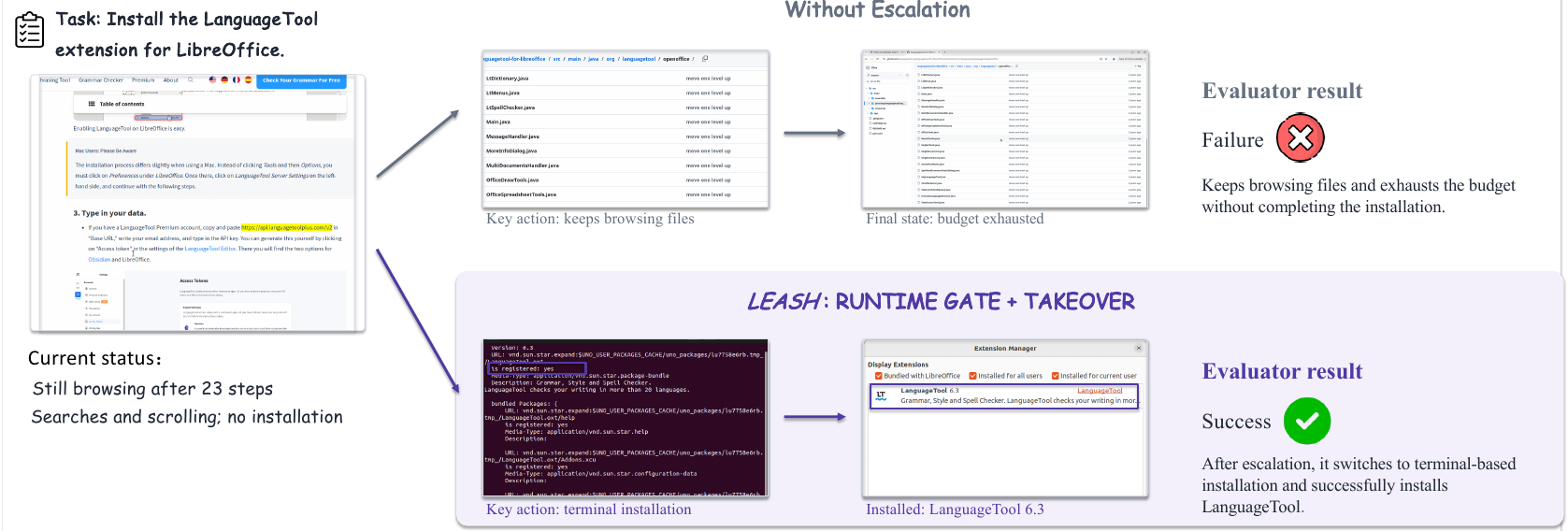}
    \caption{\textbf{Runtime escalation case study III: extension installation.} The weak agent continues browsing without completing the installation. After handoff at step 24, the strong agent switches to a terminal-based installation strategy and successfully installs the requested extension.}
    \label{fig:runtime_case_3}
\end{figure*}

\end{document}